\documentclass[letterpaper,10pt,conference]{ieeeconf}

\IEEEoverridecommandlockouts
\usepackage{amsmath}
\usepackage{amssymb}
\usepackage{algorithm}
\usepackage{algpseudocode}
\usepackage{array}
\usepackage{booktabs}
\usepackage{cite}
\usepackage{graphicx}
\usepackage{microtype}
\usepackage{multirow}
\usepackage{placeins}
\usepackage{url}

\newcommand{\R}{\mathbb{R}}
\newcommand{\cN}{\mathcal{N}}
\newcommand{\cU}{\mathcal{U}}
\newcommand{\ubar}{\bar{u}}
\newtheorem{theorem}{Theorem}
\newtheorem{lemma}{Lemma}
\newtheorem{proposition}{Proposition}
\newtheorem{corollary}{Corollary}
\newtheorem{assumption}{Assumption}

\title{\LARGE \bf Safety-Filtered Distributed Koopman-MPC}

\author{Shengjun~Zhang$^{1,*}$, Wenhao~Li$^{2,3,*}$,
Zhenxin~Lin$^{4}$, and Zhenglong~Sun$^{2,3,\dagger}$\\[-1pt]
{\footnotesize $^{1}$School of Artificial Intelligence, Hubei University,
Wuhan, Hubei, 430061 China}\\[-2pt]
{\footnotesize $^{2}$Shenzhen Institute of Artificial Intelligence and Robotics for
Society (AIRS), Shenzhen, Guangdong, China}\\[-2pt]
{\footnotesize $^{3}$School of Science and Engineering, The Chinese University of
Hong Kong, Shenzhen (CUHK-Shenzhen), Shenzhen, Guangdong, China}\\[-2pt]
{\footnotesize $^{4}$Manchester Metropolitan Joint Institute, Hubei University,
Wuhan, China}\\[-2pt]
{\scriptsize \texttt{sj.zhang@hubu.edu.cn},
\texttt{wenhaoli1@link.cuhk.edu.cn}, \texttt{linzhenxin@stu.hubu.edu.cn},
\texttt{sunzhenglong@cuhk.edu.cn}}\\[-2pt]
{\scriptsize $^*$Equal contribution. $^\dagger$Corresponding author.}}

\begin{document}

\maketitle
\vspace{-4pt}
\thispagestyle{plain}
\pagestyle{plain}

\begin{abstract}
Distributed model predictive control (DMPC) often constructs both predictions
and collision constraints from neighbor trajectories, so packet loss can remove
both. We separate these roles: received trajectories drive Koopman-MPC, while
local sensing and shelf geometry define a hard-constrained quadratic program
(QP) that projects the applied input. Its radial demand is the least constant acceleration that keeps a
supporting-plane clearance nonnegative throughout one zero-order-hold interval.
Complementary pair rows recover the coupled demand without exchanging safety
decisions. We give an intersample separation theorem under bounded snapshot and
directional plant errors, an exact max--min test for simultaneous local
feasibility, and a sensing-radius condition for switching interaction graphs. Anticipatory
high-order rows may be relaxed for performance, but the finite-hold rows contain
no safety slack. Matched eight-robot warehouse simulations use a frozen Koopman
model, nonlinear drift, bounded inputs and speed, shelf constraints, a
$120~\mathrm{ms}$ control period, and packet dropout. The full controller is
collision-free in 20/20 matched trials and reaches 160/160 robot goals;
predictive Koopman-MPC without the final projection is collision-free in 1/20
trials. All 38,400 full-method hard-row sets pass the online feasibility test,
and every local QP solves. Five-stream fleet sweeps are collision-free and
hard-row feasible through 16 robots; the 20-robot boundary fails only after the
online margin turns negative, while the reconstructed per-agent critical path
remains below the sampling period. Bounded-sensing and differential-drive
tests provide additional deployment stress.
\end{abstract}

\section{Introduction}

Distributed model predictive control (DMPC) often uses received neighbor
trajectories for both horizon planning and pairwise collision constraints; a
lost packet can therefore affect both. Reactive velocity-obstacle methods use
current geometry~\cite{fiorini1998,van2011orca}, whereas DMPC plans over a
horizon~\cite{camponogara2002,dunbar2006}. We retain horizon planning but construct
the final collision constraints independently of message delivery.

The baseline exposes the failure mode: an eight-robot Koopman-MPC planner reaches
its goals accurately but collides in nearly every packet-loss run. Prediction
supports task progress, yet a missing plan removes the associated pair
constraint. We therefore use exchanged trajectories for nominal planning and a
current local snapshot for constraints on the applied input.

Koopman lifting yields a linear predictor and a condensed quadratic program (QP)
for local MPC~\cite{korda2018koopman}. A second local QP projects the first input
onto sensed pair and shelf constraints. Its hard row is derived for the actual
sample period and a bounded plant residual. For each pair $(i,j)$, the robots
enforce complementary shares of one relative-acceleration demand; the two rows
sum to the coupled condition without centralized or iterative safety
optimization.

This paper makes three contributions.
\begin{enumerate}
    \item A two-channel Koopman-DMPC architecture in which trajectory packets
    affect planning, but every final pair constraint comes from local sensing.
    \item A hard finite-hold projection with bounded-state-error tightening,
    an intersample separation proof, complementary authority allocation, an
    exact multi-row feasibility margin, and a switching-graph condition.
    \item Matched simulation tests with a frozen model, analytic drift bounds,
    hard-QP audits, unequal actuation limits, packet-loss, sensing-error,
    differential-drive, layout, and fleet-size stress tests.
\end{enumerate}
The predictor shapes nominal motion but does not enter the hard feasible set.
Accordingly, the theory covers collision avoidance, not learned-model stability
or global liveness. The experiments report collision avoidance and route-level
completion as separate outcomes.

\section{Related Work}

\textit{Koopman prediction and control.}
Extended dynamic mode decomposition (EDMD) and dynamic mode decomposition with
control (DMDc) learn finite-dimensional predictors from nonlinear
data~\cite{williams2015edmd,proctor2016dmdc}. Koopman predictors support linear
and robust MPC~\cite{korda2018koopman,zhang2022robustkoopman,
bruder2021koopman}, including multi-robot guidance~\cite{zhao2024} and
distributed obstacle-density prediction~\cite{azarbahram2025}. Our frozen
Koopman model supplies the nominal planner. Prediction error can change the
required correction, but it cannot remove a sensed safety row.

\textit{Distributed predictive coordination.}
DMPC spans decomposed receding-horizon control and iterative coordination
~\cite{camponogara2002,dunbar2006,scattolini2009}; learned forecasting and games
also support crowd navigation~\cite{le2023}. Recent DMPC with control barrier
functions (CBFs) uses node--edge alternating direction method of multipliers
(ADMM) updates to optimize a coupled horizon in 15 reported rounds per
cycle~\cite{zeng2026}. Our horizon optimizer is local and nominal; one local QP
enforces current-step safety. This sacrifices coupled horizon optimality but
removes edge copies, dual variables, stopping tests, and safety messages.
Warehouse dispatch and multi-agent path finding (MAPF)
~\cite{wurman2008,honig2019} provide higher-level routes.

\textit{Reactive safety.}
Optimal reciprocal collision avoidance (ORCA) uses reciprocal half-planes in
velocity space~\cite{van2011orca,alonso2013}. CBFs encode forward-invariance
conditions as optimization constraints~\cite{ames2017cbf},
including high-relative-degree constraints~\cite{nguyen2016ecbf} and
multi-robot barrier formulations~\cite{wang2017barrier}. Distributed CBF
implementations~\cite{tan2022}, optimization-based safe
navigation~\cite{mestres2024}, and delay-aware barriers~\cite{ballotta2024}
address distributed enforcement, coupling, and communication delay. Symmetric and
asymmetric pairwise CBF allocations exist; the latter can avoid
relative-velocity information~\cite{dan2025}. Predictive safety filters (PSFs)
likewise separate a task policy from a safety-filtered correction
~\cite{wabersich2021psf}. Sampled-data CBF analyses explicitly account for
zero-order hold and input delay~\cite{singletary2020}.
Table~\ref{tab:mechanisms} summarizes the differences. ORCA projects velocity;
our filter projects acceleration after horizon control. PSFs use a backup
horizon; ours uses a current-step radial lower bound evaluated over the full
hold. Unlike a softened barrier row, the hard row is not relaxed. The auxiliary
variable affects only anticipatory braking.

\begin{table}[t]
\centering
\caption{Comparison of online safety mechanisms.}
\label{tab:mechanisms}
\scriptsize
\begin{tabular}{@{}>{\raggedright\arraybackslash}p{0.15\columnwidth}
                    >{\raggedright\arraybackslash}p{0.25\columnwidth}
                    >{\raggedright\arraybackslash}p{0.49\columnwidth}@{}}
\toprule
Method & Safety decision & Online coupling and role \\
\midrule
ORCA~\cite{van2011orca} & preferred velocity & one-shot reciprocal velocity half-planes; reactive avoidance \\
PSF~\cite{wabersich2021psf} & backup horizon & model-based predictive backup and recovery \\
CBF-DMPC~\cite{zeng2026} & state--input horizon & iterative node--edge consensus toward a coupled solution \\
Ours & held acceleration & independent local QPs; hard finite-hold row and exact local feasibility test \\
\bottomrule
\end{tabular}
\end{table}

\section{Problem Formulation}

Let $t_k=kT$ denote the sampling instants, and index the $N$ robots by
$i\in\mathcal V:=\{1,\ldots,N\}$. Robot $i$ has state
$x_i=[p_i^\top,v_i^\top]^\top\in\R^4$ and applies a constant command
$u_{i,k}\in\cU_i\subset\R^2$ on $[t_k,t_{k+1})$. Its position evolves as
\begin{equation}
 \dot p_i=v_i,\qquad \dot v_i=u_{i,k}+d_i(t),
 \quad t\in[t_k,t_{k+1}),
 \label{eq:dynamics}
\end{equation}
where $d_i$ collects plant drift and acceleration mismatch. The compact convex
set $\cU_i$ contains the applied-command limits; a separate speed bound
$\|v_i(t)\|\le \bar v_i$ is used only to reason about sensing activation.
We write $p_{i,k}=p_i(t_k)$ and $v_{i,k}=v_i(t_k)$ for sampled states.
Position is continuous at sampling instants. If a low-level speed limiter updates
the velocity at an instant $t_k$, the controller uses that measured post-update
velocity as the initial condition of the next hold.
The control objective is to move each robot toward $g_i\in\R^2$ while
maintaining
\begin{equation}
 \|p_i-p_j\|\ge d_{\rm col},\qquad
 \operatorname{dist}(p_i,\mathcal P_o)\ge r_{\rm rob}
 \label{eq:safe-set}
\end{equation}
for every pair and every convex shelf or pallet $\mathcal P_o$.

Robot $i$ receives a predicted trajectory from $j$ only when $j$ lies inside
communication radius $R_{\rm comm}$ and the packet is delivered; $\cN_i^c(k)$
denotes this random message neighborhood. Separately, the local sensor supplies
snapshots $\hat p_{i,k},\hat v_{i,k}$ with known errors
$\|\hat p_{i,k}-p_{i,k}\|\le\varepsilon_i^p$ and
$\|\hat v_{i,k}-v_{i,k}\|\le\varepsilon_i^v$. It activates
$\cN_i^s(k)=\{j\ne i:\|\hat p_{i,k}-\hat p_{j,k}\|\le R_{\rm sens}\}$.
Pair snapshots are synchronized and symmetric: both members use one relative
snapshot, opposite normals, and the same tightened demand. Let $\mathcal O_i(k)$
denote the active shelf faces, whose geometry is locally available. Using the
snapshot-based normals defined below, we assume known directional bounds
\begin{equation}
\begin{aligned}
 n_{ij,k}^\top[d_i(t)-d_j(t)]&\ge-\bar d_{ij},\\
 n_{io,k}^\top d_i(t)&\ge-\bar d_{io},
\end{aligned}\quad t\in[t_k,t_{k+1}).
\label{eq:directional-bounds}
\end{equation}
These are directional bounds on the plant, not learned-model error bounds. We
derive numerical values for the simulated dynamics in
Section~\ref{sec:experiments}. Packet loss affects trajectory exchange but not
the local snapshot used to build the applied-input constraints. Our safety
results concern~\eqref{eq:safe-set}; Koopman prediction accuracy and goal
arrival are evaluated separately.

\section{Method}

Fig.~\ref{fig:pipeline} shows where the two information streams enter the
controller. Received plans enter only the blue MPC layer. The red QP is rebuilt
from current relative states and shelf geometry, then filters the nominal input
before it is applied to the plant.

\begin{figure}[t]
    \centering
    \includegraphics[width=0.96\columnwidth]{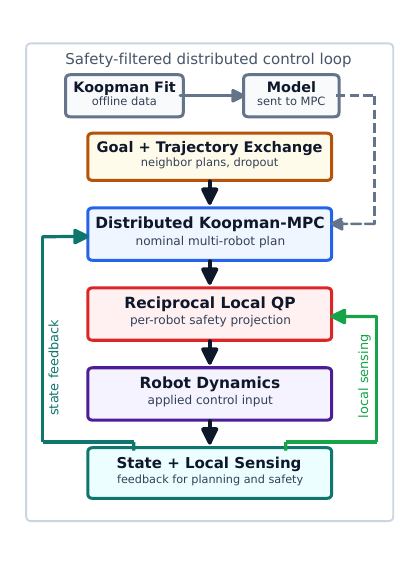}
    \caption{Controller architecture. Each robot solves a Koopman-MPC from
    received trajectories, followed by an independent reciprocal QP from local
    sensing. Packet loss can change the nominal input but does not remove a
    finite-hold row for a sensed pair.}
    \label{fig:pipeline}
\end{figure}

\subsection{Koopman Identification}

We fit one extended dynamic mode decomposition with control (EDMDc) predictor
offline and keep it fixed during evaluation. Let
$z=\psi(x)\in\R^{n_z}$ be a dictionary of observables. For $M$ transitions,
define $Z_-=[\psi(x_0)\ \cdots\ \psi(x_{M-1})]$, $Z_+=[\psi(x_1)\ \cdots\
\ \psi(x_M)]$, $U=[u_0\ \cdots\ u_{M-1}]$, and
$\Omega=[Z_-^\top\ U^\top]^\top$, with
$Z_\pm\in\R^{n_z\times M}$ and $U\in\R^{2\times M}$. EDMDc identifies
\begin{equation}
 z_{k+1}=Az_k+Bu_k,\qquad \hat p_k=C_pz_k,
 \label{eq:koopman}
\end{equation}
by ridge regression with $\epsilon>0$,
\begin{equation}
 [A\ B]=Z_+\Omega^\top(\Omega\Omega^\top+\epsilon I)^{-1}.
\end{equation}
The lift is $\psi(x)=[1;x;\operatorname{vech}(xx^\top);\phi_\kappa(p)]$, where
$\phi_\kappa(p)=[\sin(\kappa p_x),\cos(\kappa p_x),\sin(\kappa p_y),
\cos(\kappa p_y)]^\top$.
$C_p$ and $C_v$ select physical position and velocity. Training covers the
operating envelope, and no training rollout is reused in evaluation.

\subsection{Local Predictive Planning}

At control step $k$, robot $i$ receives neighbor predictions
$\hat p_{j,1:H}$ for $j\in\cN_i^c(k)$. Its local reference $r_{i,\tau}$ is the
goal $g_i$, or a waypoint when a route planner is present. Linearizing pair and
shelf directions about the previous plans gives the local QP
\begin{equation}
\label{eq:mpc}
\begin{aligned}
 \min_{z,u,\xi\ge0}\quad &
 \sum_{\tau=1}^{H}\big(\|C_pz_\tau-r_{i,\tau}\|_{Q_\tau}^{2}
 +\|C_vz_\tau\|_{S}^{2}\big)\\
 &+\sum_{\tau=0}^{H-1}\|u_\tau\|_{R}^{2}
 +\sum_{\tau=1}^{H-1}\|u_\tau-u_{\tau-1}\|_{R_\Delta}^{2}
 +\rho_p\|\xi\|^2\\
 \text{s.t.}\quad &z_0=\psi(x_{i,k}),\quad
 z_{\tau+1}=Az_\tau+Bu_\tau,\\
 &n_{ij,\tau}^\top(C_pz_\tau-\hat p_{j,\tau})+\xi_{ij,\tau}
 \ge d_{\rm plan},\\
 &n_{io,\tau}^\top(C_pz_\tau-q_{o,\tau})+\xi_{io,\tau}
 \ge d_o,\\
 &u_\tau\in\cU_i^{\rm box},\quad C_vz_\tau\in\mathcal V_i.
\end{aligned}
\end{equation}
Here $H$ is the horizon, $\xi$ collects predictive slacks,
$Q_\tau\succeq0$, $S\succeq0$, $R\succ0$,
$R_\Delta\succeq0$, and $\rho_p>0$. The sets $\cU_i^{\rm box}$ and
$\mathcal V_i$ impose componentwise planner bounds. Dynamics and input rows use
$\tau=0{:}H-1$; separation rows use $\tau=1{:}H$, received neighbors
$j\in\cN_i^c(k)$, and active faces $o\in\mathcal O_i(k)$. The vectors
$n_{ij,\tau}$ are unit separation directions; $q_{o,\tau}$ is the closest point
on active shelf face $o$, $n_{io,\tau}$ is its outward normal, and $d_o$ is the
planner shelf clearance. After norm clipping, the first MPC input becomes the
nominal input $\ubar_{i,k}$. If a packet is missing, its predictive pair rows
disappear from~\eqref{eq:mpc}. The planning QP therefore defines nominal motion
but does not impose the sensed constraints on the applied input.

\subsection{Finite-Hold Local Safety Projection}

Use the bounded local snapshots from Section~III; exact sensing has
$\varepsilon_i^p=\varepsilon_i^v=0$. For a sensed pair, set
$n_{ij,k}=(\hat p_{i,k}-\hat p_{j,k})/
\|\hat p_{i,k}-\hat p_{j,k}\|$ and define the tightened quantities
\begin{equation}
\begin{aligned}
 h_{ij,k}&=\|\hat p_{i,k}-\hat p_{j,k}\|-d_{\rm col}
            -\varepsilon_i^p-\varepsilon_j^p,\\
 \nu_{ij,k}&=n_{ij,k}^\top(\hat v_{i,k}-\hat v_{j,k})
            -\varepsilon_i^v-\varepsilon_j^v.
\end{aligned}
 \label{eq:pair-clearance}
\end{equation}
The triangle and Cauchy--Schwarz inequalities give
$n_{ij,k}^\top(p_{i,k}-p_{j,k})-d_{\rm col}\ge h_{ij,k}$ and
$n_{ij,k}^\top(v_{i,k}-v_{j,k})\ge\nu_{ij,k}$.
For $h\ge0$, introduce the least constant radial acceleration that prevents the
quadratic lower bound from becoming negative during one hold:
\begin{equation}
 \phi_T(h,\nu)=\sup_{0<\tau\le T}
 -\frac{2(h+\tau\nu)}{\tau^2}.
 \label{eq:hold-demand}
\end{equation}
The corresponding closed form is
\begin{equation}
 \phi_T(h,\nu)=\begin{cases}
 \nu^2/(2h), & h>0,\ \nu\le-2h/T,\\
 -2(h+T\nu)/T^2, & h>0,\ \nu>-2h/T,\\
 -2\nu/T, & h=0,\ \nu\ge0,\\
 +\infty, & h=0,\ \nu<0.
 \end{cases}
\label{eq:hold-demand-closed}
\end{equation}
The robust hard demand is
\begin{equation}
 \beta^{\rm H}_{ij,k}=\phi_T(h_{ij,k},\nu_{ij,k})+\bar d_{ij}.
 \label{eq:hard-demand}
\end{equation}

The hard row reacts only to motion that can violate separation within the
current hold. We retain the high-order expression as an anticipatory target.
Let $k_v,k_p>0$ and define the directional authorities
\begin{equation}
 c_{ij}^i=\max_{u\in\cU_i}n_{ij,k}^\top u,\qquad
 c_{ij}^j=\max_{u\in\cU_j}n_{ji,k}^\top u.
 \label{eq:authority}
\end{equation}
Let $b^{\rm c}_{ij}=-k_v\nu_{ij,k}-k_ph_{ij,k}+\bar d_{ij}$. Then
\begin{equation}
 \beta^{\rm A}_{ij,k}=\max\!\left\{\beta^{\rm H}_{ij,k},
 \min\{b^{\rm c}_{ij},c_{ij}^i+c_{ij}^j\}\right\}.
 \label{eq:anticipatory-demand}
\end{equation}
The cap prevents an isolated pair target from exceeding its combined radial
authority. The input sets are known to both members of a sensed pair. For
heterogeneous robots we use complementary weights
\begin{equation}
 \alpha_{ij}=\begin{cases}
 \dfrac{c_{ij}^i}{c_{ij}^i+c_{ij}^j},&c_{ij}^i+c_{ij}^j>0,\\[2mm]
 \dfrac12,&c_{ij}^i+c_{ij}^j=0,
 \end{cases}
 \qquad \alpha_{ji}=1-\alpha_{ij};
 \label{eq:authority-split}
\end{equation}
equal robots recover the half split.

For obstacle $o$, take a supporting plane $(q_{io,k},n_{io,k})$ at the
estimated position and set
$h_{io,k}=n_{io,k}^\top(\hat p_{i,k}-q_{io,k})-r_{\rm rob}-\varepsilon_i^p$ and
$\nu_{io,k}=n_{io,k}^\top\hat v_{i,k}-\varepsilon_i^v$. We define
$\beta^{\rm H}_{io,k}=\phi_T(h_{io,k},\nu_{io,k})+\bar d_{io}$ and
\begin{equation}
 \beta^{\rm A}_{io,k}=\max\{\beta^{\rm H}_{io,k},
 \min(-k_v\nu_{io,k}-k_ph_{io,k}+\bar d_{io},c_{io}^i)\},
\end{equation}
where $c_{io}^i=\max_{u\in\cU_i}n_{io,k}^\top u$.

At a fixed sample, we suppress the index $k$ in the following QP. Robot $i$
solves for $u_i$ and $\zeta_i$. The variables $\zeta_i$ relax only the
anticipatory targets; they do not appear in the hard rows.
\begin{equation}
\label{eq:localqp}
\begin{aligned}
 \min_{u_i,\zeta_i\ge0}\quad &
 \|u_i-\ubar_i\|^2+\rho_2\|\zeta_i\|^2
 +\rho_1\mathbf1^\top\zeta_i\\
 \mathrm{s.t.}\quad &u_i\in\cU_i,\\
 &n_{ij}^\top u_i\ge\alpha_{ij}\beta^{\rm H}_{ij},
 &&j\in\cN_i^s,\\
 &n_{ij}^\top u_i+\zeta_{ij}^i
 \ge\alpha_{ij}\beta^{\rm A}_{ij},&&j\in\cN_i^s,\\
 &n_{io}^\top u_i\ge\beta^{\rm H}_{io},
 &&o\in\mathcal O_i,\\
 &n_{io}^\top u_i+\zeta_{io}^i\ge\beta^{\rm A}_{io},
 &&o\in\mathcal O_i.
\end{aligned}
\end{equation}
Thus a large $\zeta_i$ means that early braking was sacrificed to remain close
to the MPC command; it does not weaken the finite-hold condition. We use a
16-sided polygon inscribed in the acceleration ball for $\cU_i$ and solve the
QP with the Operator Splitting Quadratic Program solver (OSQP)
~\cite{stellato2020osqp}. The weights $\rho_1,\rho_2>0$ penalize
anticipatory relaxation.

\subsection{Online Decomposition and Feasibility Check}

Let the hard rows of robot $i$ be $a_{i\ell}^\top u_i\ge r_{i\ell}$,
$\ell\in\mathcal L_i(k)$, and define
\begin{equation}
 \mu_i(k)=\max_{u\in\cU_i}\min_{\ell\in\mathcal L_i(k)}
 (a_{i\ell}^\top u-r_{i\ell}).
 \label{eq:feasibility-margin}
\end{equation}
where $\mu_i(k)=+\infty$ when $\mathcal L_i(k)$ is empty.
For polygonal $\cU_i$, this is a linear program in $(u_x,u_y,\mu)$. We solve it
before~\eqref{eq:localqp} and log a feasibility fault when $\mu_i<0$. In two
dimensions its vertices consist of input-polygon vertices, intersections of one
input facet with two row planes, and intersections of three row planes; the
implementation evaluates these candidates in vectorized form. No safety slack
is used in the hard rows; a negative $\mu_i$ is reported as infeasibility. With
$m_i$ local rows, the
projection has at most $2+m_i$ variables and $16+3m_i$ scalar inequalities. Its
size is independent of the MPC horizon and the fleet size outside the sensed
degree.

\begin{algorithm}[H]
\caption{Distributed Koopman-MPC with finite-hold projection}
\label{alg:online}
\begin{algorithmic}[1]
\Require frozen predictor; sampling, sensing, and plant bounds
\For{each sample $k$}
    \For{each robot $i$ in parallel}
        \State receive available plans and solve~\eqref{eq:mpc} for $\ubar_{i,k}$
        \State form bounded local snapshots and supporting obstacle planes
        \State build $\beta^{\rm H}$ from~\eqref{eq:hold-demand} and $\beta^{\rm A}$
        \State solve~\eqref{eq:feasibility-margin}; require $\mu_i(k)\ge0$
        \State solve~\eqref{eq:localqp} and hold $u_{i,k}$ for $T$ seconds
    \EndFor
\EndFor
\end{algorithmic}
\end{algorithm}

\section{Theoretical Results}
\label{sec:theory}

The Koopman predictor determines $\ubar_i$ but not the hard feasible set. The
following results therefore make no claim about predictor error, MPC recursive
feasibility, or goal convergence. Assume $\cU_i$ is nonempty, compact, and
convex, $0\in\cU_i$, and $\rho_1,\rho_2>0$.

\begin{assumption}
\label{ass:online}
At each $t_k$, the snapshot errors satisfy the stated bounds. Pair sensing is
synchronized and symmetric, so both members use the same $h_{ij,k}$,
$\nu_{ij,k}$, and $\beta^{\rm H}_{ij,k}$ with opposite normals. Every obstacle
row uses a unit outward normal of a supporting plane, and the
directional residual bounds in~\eqref{eq:directional-bounds} hold throughout
$[t_k,t_{k+1})$. Commands returned by~\eqref{eq:localqp} are applied with
zero-order hold.
\end{assumption}

\begin{proposition}
\label{prop:residual-support}
Let $\mathcal D_i(k)\subset\R^2$ be compact, contain the origin, and satisfy
$d_i(t)\in\mathcal D_i(k)$ throughout the hold. For the support function
$\sigma_{\mathcal D}(q):=\max_{d\in\mathcal D}q^\top d$, the choices
\begin{equation}
\begin{aligned}
 \bar d_{ij,k}&=\sigma_{\mathcal D_i(k)}(-n_{ij,k})
               +\sigma_{\mathcal D_j(k)}(n_{ij,k}),\\
 \bar d_{io,k}&=\sigma_{\mathcal D_i(k)}(-n_{io,k})
\end{aligned}
\label{eq:support-bounds}
\end{equation}
satisfy~\eqref{eq:directional-bounds} and are tight over
$\mathcal D_i(k)\times\mathcal D_j(k)$. Moreover, if
$\mathcal D_i(k)=\bigoplus_{r=1}^{s}\mathcal D_i^{(r)}(k)$, then
\begin{equation}
 \sigma_{\mathcal D_i(k)}(q)
 =\sum_{r=1}^{s}\sigma_{\mathcal D_i^{(r)}(k)}(q).
 \label{eq:minkowski-support}
\end{equation}
Consequently, bounds $\|d_i^{(r)}\|\le\delta_i^{(r)}$ imply
$\bar d_{io,k}\le\sum_r\delta_i^{(r)}$ and
$\bar d_{ij,k}\le\sum_r(\delta_i^{(r)}+\delta_j^{(r)})$.
\end{proposition}

\begin{proof}
For any $d_i\in\mathcal D_i(k)$ and $d_j\in\mathcal D_j(k)$,
\begin{align*}
 n_{ij,k}^\top(d_i-d_j)
 &=-(-n_{ij,k})^\top d_i-n_{ij,k}^\top d_j\\
 &\ge-\sigma_{\mathcal D_i(k)}(-n_{ij,k})
       -\sigma_{\mathcal D_j(k)}(n_{ij,k}).
\end{align*}
Replacing $n_{ij,k}$ by $n_{io,k}$ and omitting $d_j$ proves the obstacle
inequality. Compactness supplies support maximizers; choosing them independently
attains equality in both bounds, which proves tightness. The origin assumption
makes the bounds nonnegative.

For the Minkowski sum, independent component choices give
\begin{align*}
 \sigma_{\mathcal D_i(k)}(q)
 &=\max_{d_i^{(r)}\in\mathcal D_i^{(r)}(k)}
   q^\top\sum_{r=1}^{s}d_i^{(r)}\\
 &=\sum_{r=1}^{s}\max_{d_i^{(r)}\in\mathcal D_i^{(r)}(k)}
   q^\top d_i^{(r)},
\end{align*}
which is~\eqref{eq:minkowski-support}. A Euclidean ball of radius $\delta$ has
$\sigma(q)=\delta\|q\|$ by Cauchy--Schwarz. Substitution into
\eqref{eq:support-bounds}, with unit normals, yields the norm-bound formulas.
\end{proof}

\begin{lemma}
\label{lem:hold-demand}
For $h\ge0$,~\eqref{eq:hold-demand-closed} equals the supremum in
\eqref{eq:hold-demand}. Whenever this value is finite,
\begin{equation}
 h+\tau\nu+\tfrac12\tau^2a\ge0\quad\forall\tau\in[0,T]
 \quad\Longleftrightarrow\quad a\ge\phi_T(h,\nu).
 \label{eq:hold-demand-equivalence}
\end{equation}
\end{lemma}

\begin{proof}
For $0<\tau\le T$, let
$f(\tau)=-2(h+\tau\nu)/\tau^2$. Since $\tau^2/2>0$, the inequality on the
left of~\eqref{eq:hold-demand-equivalence} holds for every positive $\tau$ if
and only if $a\ge f(\tau)$ for every such $\tau$. Thus its least admissible
value is $\sup_{0<\tau\le T}f(\tau)$; continuity supplies the endpoint
$\tau=0$.

For $h>0$, $f(\tau)\to-\infty$ as $\tau\downarrow0$ and
\begin{equation*}
 f'(\tau)=\frac{2(2h+\nu\tau)}{\tau^3}.
\end{equation*}
If $\nu\le-2h/T$, the stationary point
$\tau_*=-2h/\nu\in(0,T]$ is the maximizer and
$f(\tau_*)=\nu^2/(2h)$. Otherwise $f$ is increasing at every point up to its
maximum at $T$, which gives $-2(h+T\nu)/T^2$. For $h=0$,
$f(\tau)=-2\nu/\tau$: it is unbounded above when $\nu<0$, and its maximum is
$-2\nu/T$ when $\nu\ge0$. These are precisely the four cases in
\eqref{eq:hold-demand-closed}, and the equivalence follows from the supremum
argument above.
\end{proof}

\begin{lemma}
\label{lem:composition}
For a symmetrically sensed pair with $\alpha_{ij}+\alpha_{ji}=1$, the two local
hard rows imply
\begin{equation}
 n_{ij,k}^\top(u_{i,k}-u_{j,k})\ge\beta^{\rm H}_{ij,k}.
 \label{eq:coupled-hard-row}
\end{equation}
\end{lemma}

\begin{proof}
At sample $k$, robot $i$ enforces
$n_{ij,k}^\top u_{i,k}\ge\alpha_{ij}\beta^{\rm H}_{ij,k}$, while robot $j$
enforces
$n_{ji,k}^\top u_{j,k}\ge\alpha_{ji}\beta^{\rm H}_{ij,k}$.
Since $n_{ji,k}=-n_{ij,k}$, adding the rows gives
\begin{equation*}
 n_{ij,k}^\top(u_{i,k}-u_{j,k})
 \ge(\alpha_{ij}+\alpha_{ji})\beta^{\rm H}_{ij,k},
\end{equation*}
which is~\eqref{eq:coupled-hard-row}.
\end{proof}

\begin{proposition}
\label{prop:wellposed}
The hard rows of robot $i$ are jointly feasible if and only if $\mu_i(k)\ge0$.
Whenever this holds,~\eqref{eq:localqp} has a unique optimizer. Its feasible set
is independent of trajectory packets.
\end{proposition}

\begin{proof}
If $\mathcal L_i(k)$ is empty, every $u\in\cU_i$ satisfies the hard rows by
convention. Otherwise, compactness of $\cU_i$ and continuity of the pointwise
minimum make the maximum in~\eqref{eq:feasibility-margin} attainable. Let
$u_i^\star$ be a maximizer. If $\mu_i(k)\ge0$, then
\begin{equation*}
 a_{i\ell}^\top u_i^\star-r_{i\ell}\ge\mu_i(k)\ge0
 \qquad\forall\ell\in\mathcal L_i(k),
\end{equation*}
so every hard row is satisfied. Conversely, any hard-row-feasible $u$ makes
the inner minimum nonnegative; maximizing over $u$ yields $\mu_i(k)\ge0$.

Fix a hard-row-feasible $u$. Each anticipatory inequality can be satisfied by
choosing its component of $\zeta_i$ no smaller than the positive part of its
row deficit. Hence~\eqref{eq:localqp} is feasible. Its feasible set is closed
and convex. The quadratic objective has positive-definite Hessian on
$(u_i,\zeta_i)$ because $\rho_2>0$; it is also coercive in $\zeta_i$, while
$u_i$ lies in the compact set $\cU_i$. A minimizer therefore exists and strict
convexity makes it unique. Finally, trajectory packets change only $\ubar_i$ in
the objective. Every constraint is built from the local snapshot, geometry,
input sets, and plant bounds.
\end{proof}

\begin{proposition}
\label{prop:margin-dual}
Suppose $\mathcal L_i(k)$ is nonempty and
$\cU_i=\{u:F_i u\le g_i\}$ is a compact polytope. Then
\begin{equation}
\begin{aligned}
 \mu_i(k)=\min_{\lambda,\gamma}\quad
 &g_i^\top\lambda-\sum_{\ell\in\mathcal L_i(k)}\gamma_\ell r_{i\ell}\\
 \mathrm{s.t.}\quad
 &F_i^\top\lambda=\sum_{\ell\in\mathcal L_i(k)}\gamma_\ell a_{i\ell},\\
 &\mathbf1^\top\gamma=1,\qquad \lambda\ge0,\quad\gamma\ge0.
\end{aligned}
\label{eq:margin-dual}
\end{equation}
\end{proposition}

\begin{proof}
Introduce an auxiliary variable $s$. Equation~\eqref{eq:feasibility-margin} is
the linear program
\begin{equation*}
 \max_{u,s}\ s\quad
 \mathrm{s.t.}\quad F_i u\le g_i,\qquad
 -a_{i\ell}^\top u+s\le-r_{i\ell}\quad\forall\ell.
\end{equation*}
Assign nonnegative dual variables $\lambda$ to the input facets and
$\gamma_\ell$ to the hard rows. Stationarity with respect to $u$ and $s$ gives
$F_i^\top\lambda=\sum_\ell\gamma_\ell a_{i\ell}$ and
$\sum_\ell\gamma_\ell=1$, respectively. The dual objective is
$g_i^\top\lambda-\sum_\ell\gamma_\ell r_{i\ell}$. The primal is feasible and
has a finite optimum because $\cU_i$ is nonempty and compact. Linear-program
strong duality therefore yields~\eqref{eq:margin-dual}.
\end{proof}

The weights $\gamma$ identify a convex combination of simultaneously limiting
rows. A negative dual value therefore exposes a local conflict among pair,
obstacle, and input constraints rather than a numerical QP failure.

\begin{proposition}
\label{prop:allocation}
For an isolated pair, the coupled hard row is feasible if and only if
$\beta^{\rm H}_{ij,k}\le c_{ij}^i+c_{ij}^j$. Under this condition, the split
in~\eqref{eq:authority-split} makes both local hard rows feasible.
\end{proposition}

\begin{proof}
Because the input sets are independent and $n_{ji,k}=-n_{ij,k}$,
\begin{align*}
 \max_{u_i\in\cU_i,u_j\in\cU_j}
 n_{ij,k}^\top(u_i-u_j)
 &=\max_{u_i\in\cU_i}n_{ij,k}^\top u_i
   +\max_{u_j\in\cU_j}n_{ji,k}^\top u_j\\
 &=c_{ij}^i+c_{ij}^j.
\end{align*}
The coupled row is therefore feasible exactly when its right-hand side does
not exceed this maximum. Since $0\in\cU_i\cap\cU_j$, both authorities are
nonnegative. If $\beta^{\rm H}_{ij,k}\le0$, $u_i=u_j=0$ satisfies both local
rows. If $\beta^{\rm H}_{ij,k}>0$, feasibility of the coupled row implies
\begin{equation*}
 \alpha_{ij}\beta^{\rm H}_{ij,k}
 \le\frac{c_{ij}^i}{c_{ij}^i+c_{ij}^j}(c_{ij}^i+c_{ij}^j)=c_{ij}^i,
\end{equation*}
and the analogous inequality holds for robot $j$. Compactness provides
maximizers in~\eqref{eq:authority}; those maximizers satisfy the two local
rows.
\end{proof}

\begin{theorem}
\label{thm:hold}
Under Assumption~\ref{ass:online}, suppose $h_{ij,k}\ge0$,
$\phi_T(h_{ij,k},\nu_{ij,k})$ is finite, and the applied commands satisfy the
two corresponding hard rows in~\eqref{eq:localqp}.
Then
\begin{equation}
 \|p_i(t_k+\tau)-p_j(t_k+\tau)\|\ge d_{\rm col}
 \quad\forall\tau\in[0,T].
 \label{eq:pair-hold-result}
\end{equation}
If $h_{io,k}\ge0$, $\phi_T(h_{io,k},\nu_{io,k})$ is finite, and the applied
command satisfies the analogous obstacle hard row, then
$\operatorname{dist}(p_i(t_k+\tau),\mathcal P_o)\ge r_{\rm rob}$ on the same
interval.
\end{theorem}

\begin{proof}
For the pair result, write $r=p_i-p_j$, $w=v_i-v_j$,
$d_{ij}=d_i-d_j$, and freeze $n=n_{ij,k}$ over the hold. By definition,
$n^\top r(t_k)-d_{\rm col}\ge h_{ij,k}$ and
$n^\top w(t_k)\ge\nu_{ij,k}$. Two integrations of~\eqref{eq:dynamics} give
\begin{align*}
 n^\top r(t_k+\tau)-d_{\rm col}
 \ge{}&h_{ij,k}+\tau\nu_{ij,k}
 +\tfrac12\tau^2n^\top(u_{i,k}-u_{j,k})\\
 &+\int_0^\tau(\tau-s)n^\top d_{ij}(t_k+s)\,ds.
\end{align*}
The directional residual bound implies
\begin{equation*}
 \int_0^\tau(\tau-s)n^\top d_{ij}(t_k+s)\,ds
 \ge-\tfrac12\tau^2\bar d_{ij}.
\end{equation*}
The two applied hard rows and Lemma~\ref{lem:composition} give
\begin{equation*}
 n^\top(u_{i,k}-u_{j,k})
 \ge\phi_T(h_{ij,k},\nu_{ij,k})+\bar d_{ij}.
\end{equation*}
Substitution cancels the residual allowance. Since $n$ is a unit vector,
$\|r\|\ge n^\top r$, and hence
\begin{align*}
\|r(t_k+\tau)\|-d_{\rm col}
&\ge n^\top r(t_k+\tau)-d_{\rm col}\\
 &\ge h_{ij,k}+\tau\nu_{ij,k}
 +\tfrac12\tau^2\phi_T(h_{ij,k},\nu_{ij,k})\\
 &\ge0.
\end{align*}
The final inequality is Lemma~\ref{lem:hold-demand}; it holds at every point
of the hold, not only at the next sample.

For the obstacle result, set $n=n_{io,k}$ and $q=q_{io,k}$. The supporting
property gives
$\mathcal P_o\subseteq\{y:n^\top(y-q)\le0\}$. Thus, for every
$y\in\mathcal P_o$,
\begin{equation*}
 \|p-y\|\ge n^\top(p-y)\ge n^\top(p-q),
\end{equation*}
and taking the infimum over $y$ yields
$\operatorname{dist}(p,\mathcal P_o)\ge n^\top(p-q)$. Integrating the
single-robot dynamics gives
\begin{align*}
 n^\top(p_i(t_k+\tau)-q)-r_{\rm rob}
 \ge{}&h_{io,k}+\tau\nu_{io,k}\\
 &+\tfrac12\tau^2
 \{n^\top u_{i,k}-\bar d_{io}\}.
\end{align*}
The obstacle hard row makes the braced term at least
$\phi_T(h_{io,k},\nu_{io,k})$. Lemma~\ref{lem:hold-demand} completes the
argument.
\end{proof}

\begin{corollary}
\label{cor:fleet}
Under Assumption~\ref{ass:online}, suppose the fleet is initially safe, every
active tightened clearance is nonnegative, every active finite-hold demand is
finite, and obstacle rows activate whenever
$\operatorname{dist}(\hat p_{i,k},\mathcal P_o)\le R_o$ using a closest
supporting plane. Suppose also that Algorithm~\ref{alg:online} applies an optimizer of
\eqref{eq:localqp} at every sample, and $\mu_i(k)\ge0$ for every active row set.
If
\begin{equation}
\begin{aligned}
 R_{\rm sens}-3(\varepsilon_i^p+\varepsilon_j^p)
 -T(\bar v_i+\bar v_j)&>d_{\rm col},\\
 R_o-3\varepsilon_i^p-T\bar v_i&>r_{\rm rob},
\end{aligned}
 \label{eq:activation-condition}
\end{equation}
where $R_o$ is the obstacle-activation distance, then the switching sensing
graph preserves~\eqref{eq:safe-set} for all continuous time. This conclusion is
independent of packet delivery.
\end{corollary}

\begin{proof}
Assume~\eqref{eq:safe-set} holds at $t_k$. For every active row,
Proposition~\ref{prop:wellposed} and $\mu_i(k)\ge0$ imply that the optimizer of
\eqref{eq:localqp} satisfies the hard constraint. Theorem~\ref{thm:hold} then
preserves the corresponding pair or obstacle separation throughout the hold.

If pair $(i,j)$ is inactive at $t_k$, then
$\|\hat p_i(t_k)-\hat p_j(t_k)\|>R_{\rm sens}$. The reverse triangle
inequality and the speed bounds give, for every $\tau\in[0,T]$,
\begin{align*}
 \|p_i(t_k+\tau)-p_j(t_k+\tau)\|
 &>R_{\rm sens}-\varepsilon_i^p-\varepsilon_j^p\\
 &\quad-\int_{t_k}^{t_k+\tau}\|v_i(s)-v_j(s)\|\,ds\\
 &\ge R_{\rm sens}-\varepsilon_i^p-\varepsilon_j^p\\
 &\quad-\tau(\bar v_i+\bar v_j)>d_{\rm col}.
\end{align*}
For an inactive obstacle, the distance function is one-Lipschitz, so
\begin{equation*}
 \operatorname{dist}(p_i(t_k+\tau),\mathcal P_o)
 >R_o-\varepsilon_i^p-\tau\bar v_i>r_{\rm rob}.
\end{equation*}
At $t_{k+1}$, a second application of the measurement-error bound gives
$\|\hat p_i-\hat p_j\|>R_{\rm sens}-2(\varepsilon_i^p+
\varepsilon_j^p)-T(\bar v_i+\bar v_j)$. Hence a newly active pair has
$h_{ij,k+1}>0$ by~\eqref{eq:activation-condition}; the one-Lipschitz obstacle
distance gives the analogous conclusion for $h_{io,k+1}$. Thus active and
inactive interactions remain safe, and newly active rows enter the next
induction step with positive tightened clearance. Packet delivery changes
neither the hard rows nor this induction.
\end{proof}

The theorem addresses intersample collision avoidance for the stated plant and
bounds. It does not turn a locally feasible safety projection into a route
planner: deadlock avoidance and goal convergence remain properties of the
nominal planning layer. The claim is conditional on the stated residual and
sensing bounds and on \(\mu_i(k)\ge0\) for every active hard-row set. If that
test fails, Algorithm~\ref{alg:online} reports infeasibility and makes no
collision-avoidance claim for the ensuing hold interval.

\FloatBarrier
\section{Experiments}
\label{sec:experiments}

\subsection{Protocol}

The main benchmark in Fig.~\ref{fig:scenario} places eight robots on opposing
warehouse lanes. Their shortest routes cross between the shelves, creating pair
and shelf conflicts at the same time. We train the Koopman model once from 220
rollouts of 22 steps using seed 7 and ridge $10^{-7}$, then freeze it.
Setting $\kappa=0.55~\mathrm{m}^{-1}$ gives the 19-dimensional lift.
Each training position coordinate lies in $[-5.5,5.5]~\mathrm{m}$;
velocities and smoothed random controls are clipped to their admissible sets.
Within each setting, all communication-dependent controllers receive the same
Bernoulli packet-loss stream. A trial seed changes that stream, not the fixed
warehouse geometry. ORCA and the reactive QP do not use these messages and are
therefore deterministic in this test. No controller uses a global route planner,
an intersection waypoint, or a resource scheduler.

\begin{figure}[t]
    \centering
    \includegraphics[width=0.96\columnwidth]{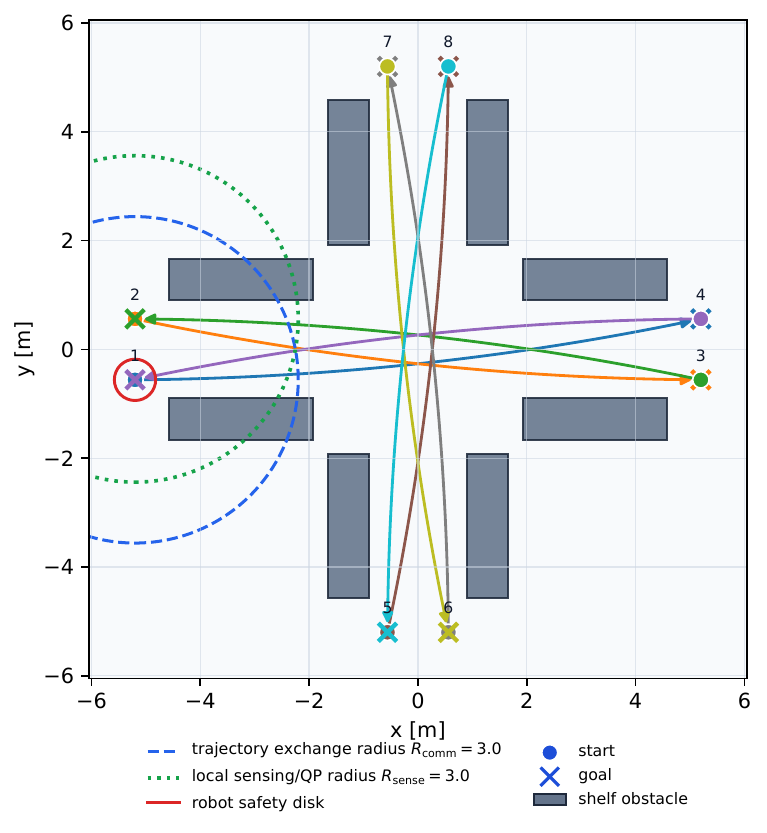}
    \caption{Eight-robot warehouse benchmark. Circles and crosses are starts
    and goals; gray rectangles are shelves. Dashed and dotted neighborhoods
    illustrate trajectory communication and local safety sensing.}
    \label{fig:scenario}
\end{figure}

We compare four controllers: (1) an ORCA half-plane linear program (LP) with
shelf constraints; (2) goal feedback followed by the same local finite-hold
projection; (3) predictive Koopman-MPC with message-dependent pair constraints
but no final projection; and (4) the full controller. A robot collision step is
a fleet time step containing any pair below
$d_{\rm col}=0.45~\mathrm{m}$. A shelf collision step contains any robot center
closer than $r_{\rm rob}=0.22~\mathrm{m}$ to a shelf. A robot run succeeds when
its final goal error is at most $0.75~\mathrm{m}$; success columns report the
fraction of successful robot runs. The sequential cycle time
includes all eight planning calls, safety QPs, and Python orchestration; a
$120~\mathrm{ms}$ control period is available, and p95 denotes the 95th
percentile. We use $d_{\min}$ for minimum pair distance and ``Shelf'' for
minimum signed center-to-shelf clearance.

\begin{table}[t]
\centering
\caption{Main simulation parameters.}
\label{tab:parameters}
\footnotesize
\begin{tabular}{@{}ll@{}}
\toprule
Parameter & Value \\
\midrule
$N$, $T$, steps & 8, $0.12~\mathrm{s}$, 240 \\
MPC horizon, $d_{\rm plan}$ & 8, $0.76~\mathrm{m}$ \\
$R_{\rm comm}$, $R_{\rm sens}$ & $3.0~\mathrm{m}$, $3.0~\mathrm{m}$ \\
Main dropout & 0.05, 0.15 \\
$d_{\rm col}$, $r_{\rm rob}$ & $0.45~\mathrm{m}$, $0.22~\mathrm{m}$ \\
$k_v$, $k_p$ & 5, 6 \\
$\bar d_{ij}$, $\bar d_{io}$ & $1.54$, $0.77~\mathrm{m/s^2}$ \\
$u_{\max}$, $v_{\max}$ & $1.8~\mathrm{m/s^2}$, $2.1~\mathrm{m/s}$ \\
\bottomrule
\end{tabular}
\end{table}

\begin{table*}[!t]
\centering
\caption{Fixed warehouse over dropout $\{0.05,0.15\}$ and ten matched streams
per level. Collision columns are mean fleet time steps per trial.}
\label{tab:main}
\footnotesize
\begin{tabular}{@{}lccccccc@{}}
\toprule
Method & Safe & Robot & Shelf & Goals & Error & $d_{\min}$ & Shelf \\
 & /20 & coll. & coll. & /160 & $[\mathrm{m}]$ & $[\mathrm{m}]$ & $[\mathrm{m}]$ \\
\midrule
ORCA & 0 & 12.0 & 63.0 & 160 & 0.631 & 0.288 & 0.125 \\
Goal + filter & 20 & 0.0 & 0.0 & 120 & 2.031 & 0.502 & 0.224 \\
Predictive MPC & 1 & 1.3 & 6.2 & 160 & 0.0065 & 0.264 & $-0.379$ \\
\textbf{Ours} & \textbf{20} & \textbf{0.0} & \textbf{0.0} & \textbf{160} & \textbf{0.0065} & \textbf{0.676} & \textbf{0.321} \\
\bottomrule
\end{tabular}
\end{table*}

\begin{figure*}[!t]
    \centering
    \includegraphics[width=0.78\textwidth]{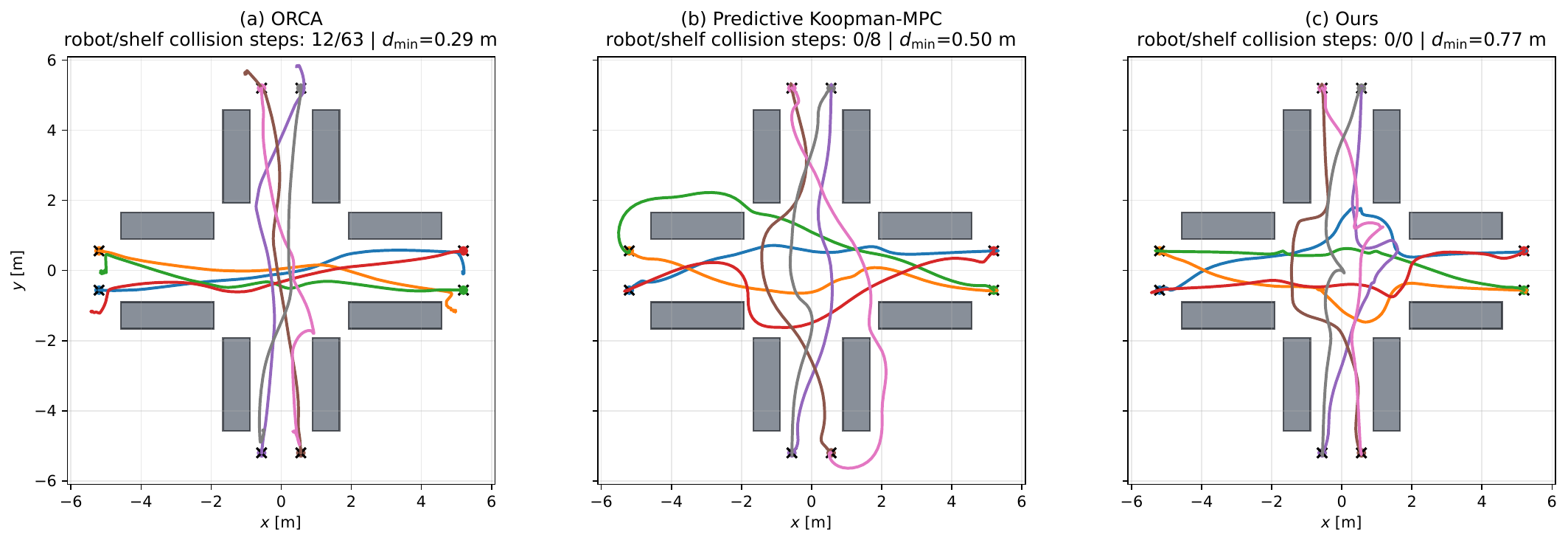}
    \caption{Representative matched trial at 15\% dropout (packet stream
    seed 3409652536). Lines are executed trajectories, colored dots are starts,
    black crosses are goals, and rectangles are shelves; Table~\ref{tab:main}
    reports all 20 trials.}
    \label{fig:trajectories}
\end{figure*}

\begin{figure*}[!t]
    \centering
    \includegraphics[width=0.96\textwidth]{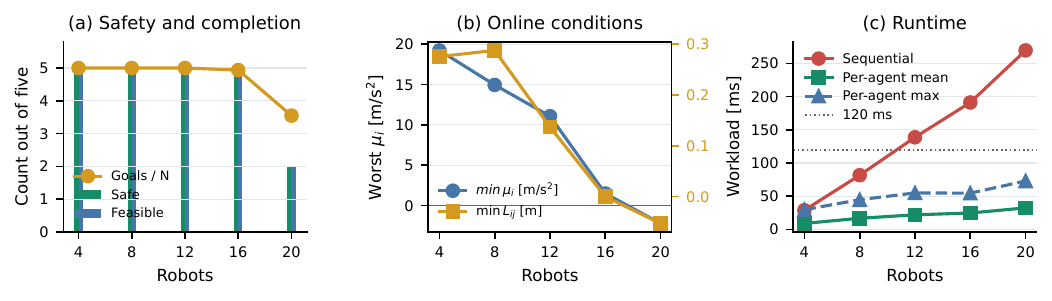}
    \caption{Fixed-footprint fleet scaling at 15\% dropout, with five packet
    streams per size. (a) Bars count safe and feasible trials; the line is
    completed robots divided by fleet size. (b) Worst online feasibility and
    finite-hold margins. (c) Central sequential workload and the reconstructed
    per-agent critical path; the dotted line is the $120~\mathrm{ms}$ period.}
    \label{fig:scaling}
\end{figure*}

The simulated residual is the sum of quadratic drag, bounded terrain terms,
and velocity cross-coupling. With $v_{\max}=2.1~\mathrm{m/s}$, each robot
satisfies
\begin{equation}
\begin{aligned}
 \|d_i\|\le{}&0.10v_{\max}^2+\sqrt{0.19^2+0.17^2}\\
 &+0.035v_{\max}<0.77~\mathrm{m/s^2}.
\end{aligned}
 \label{eq:simulator-bound}
\end{equation}
Proposition~\ref{prop:residual-support} therefore gives the obstacle bound
$\bar d_{io}=0.77~\mathrm{m/s^2}$ and pair bound
$\bar d_{ij}=2(0.77)=1.54~\mathrm{m/s^2}$. These values are fixed before any
trial. Moreover,
$R_{\rm sens}-2Tv_{\max}=2.496~\mathrm{m}>d_{\rm col}$; the
$1.8~\mathrm{m}$ obstacle trigger similarly satisfies
$1.8-Tv_{\max}>r_{\rm rob}$. Thus the reported configuration meets
\eqref{eq:activation-condition} by construction.

\subsection{Prediction Ablation}

Across three model seeds and 100 test rollouts each, 25-step errors are
$0.471\pm0.003$, $0.280\pm0.010$, $0.381\pm0.012$, and
$0.265\pm0.010~\mathrm{m}$ for the linear, trigonometric, quadratic, and full
lifts. The full lift is 44\% better than the linear one and uses no shelf
coordinates, so layout changes require no dynamics retraining within the sampled
envelope.

\subsection{Main Warehouse Result}

ORCA and predictive Koopman-MPC reach 160/160 goals but are safe in 0/20 and
1/20 trials. Goal feedback plus projection is safe in 20/20 but reaches 120/160
goals. Ours is safe in 20/20, reaches 160/160, and retains the predictive
controller's $0.0065~\mathrm{m}$ error; its worst pair and shelf distances are
$0.676$ and $0.321~\mathrm{m}$.

\subsection{Hard-Row Audit}

For each active pair we log the minimum quadratic lower bound used in the proof,
\begin{equation}
\begin{aligned}
 L_{ij,k}=\min_{0\le\tau\le T}\big[&h_{ij,k}+\tau\nu_{ij,k}\\
 &+\tfrac12\tau^2\{n_{ij,k}^\top(u_{i,k}-u_{j,k})-\bar d_{ij}\}\big],
\end{aligned}
 \label{eq:logged-lower-bound}
\end{equation}
and define $L_{io,k}$ analogously. Across 38,400 local solves, every LP and hard
QP solves, $\min\mu_i=14.79~\mathrm{m/s^2}$, and the smallest pair/shelf bounds
are $0.214/0.094~\mathrm{m}$. Anticipatory relaxation reaches
$1.523~\mathrm{m/s^2}$ but never enters a safety row.

\subsection{Boundary and Scaling Studies}

\begin{table}[t]
\centering
\caption{Frozen-model stress tests ($\mu_i$ in $\mathrm{m/s^2}$, $L_{ij}$ in
m); negative $\mu_i$ removes the theorem's premise.}
\label{tab:stress}
\scriptsize
\setlength{\tabcolsep}{3.1pt}
\renewcommand{\arraystretch}{0.93}
\begin{tabular}{@{}lrrrr@{}}
\toprule
Study & Safe & Goals & $\min\mu_i$ & $\min L_{ij}$\\
\midrule
Dropout 0.50 & 10/10 & 80/80 & 4.130 & 0.036\\
$R_{\rm sens}=3.0$ m & 10/10 & 80/80 & 11.760 & 0.148\\
$\varepsilon^p\le2$ cm & 15/15 & 120/120 & 7.214 & 0.087\\
Perturbed shelves & 19/20 & 160/160 & $-0.756$ & $-0.030$\\
Scaling, $N\le16$ & 20/20 & 199/200 & 1.496 & $\approx0$\\
Diff.\ drive, $N=4$ & 20/20 & 80/80 & 3.388 & 0.028\\
$\varepsilon^p=5$ cm & 4/5 & 38/40 & $-7.245$ & $-0.250$\\
Scaling, $N=20$ & 2/5 & 71/100 & $-2.213$ & $-0.053$\\
\bottomrule
\end{tabular}
\end{table}

All rows reuse model seed 7, and~\eqref{eq:activation-condition} holds at every
listed sensing-error level. Disk-bounded snapshot errors use
$\varepsilon_i^v=2\varepsilon_i^p~\mathrm{s^{-1}}$ and
\eqref{eq:pair-clearance}. At 3 cm, 5/5 runs remain collision-free but
$\min\mu_i=-0.979~\mathrm{m/s^2}$ with four hard-row faults.

Figure~\ref{fig:scaling} uses distinct ports, opposite goals, and static shelf
routes. The near-zero $N=16$ audit is solver tolerance. At $N=20$, negative
margin precedes nine collision samples, marking a fixed-footprint authority
boundary rather than a fleet-size guarantee.

The differential-drive stress preserves longitudinal acceleration and clips
lateral acceleration at $\|v_i\|\omega_{\max}$. Write the resulting tracking
mismatch as $d_i^{\rm dd}$. For $\|v_i\|>0$, let $e_i=v_i/\|v_i\|$ and let
$e_i^\perp$ be its lateral unit vector; then~\eqref{eq:support-bounds} uses
$-n^\top d_i^{\rm dd}\le|n^\top e_i^\perp|
[u_{\max,i}-\|v_i\|\omega_{\max}]_+$. At rest we use $u_{\max,i}$. We set
$v_{\max}=1.5~\mathrm{m/s}$ and $\omega_{\max}=2.4~\mathrm{rad/s}$. In the
separate heterogeneous-input study, authority allocation raises the worst pair
distance from 0.696 to $0.802~\mathrm{m}$ and cuts maximum anticipation
relaxation from 0.683 to $0.039~\mathrm{m/s^2}$.

\subsection{Solver Diagnostics}

All 38,400 MPC and 38,400 safety-QP solves return, with mean/p95/max iterations
of $103.6/304.1/2500$ and $86.9/131.4/9025$. The sequential simulator sums
eight robots and takes $90.1/109.0/159.4~\mathrm{ms}$. A per-agent
reconstruction including full row-build time takes $18.1/25.5/46.9~\mathrm{ms}$;
the scaling maximum is $73.0~\mathrm{ms}$ and no sample exceeds $120~\mathrm{ms}$.
It excludes scheduling, communication I/O, and hardware worst-case timing.

\section{Conclusion}

Local sensing wraps nominal Koopman-MPC in a hard finite-hold projection with
bounded-snapshot tightening. Under explicit plant, sensing, and feasibility
conditions, fixed-warehouse safety rises from 1/20 to 20/20 at the same
$0.0065~\mathrm{m}$ goal error. Across stress tests, unsafe cases follow a
negative online margin.

\begingroup
\renewcommand{\footnotesize}{\fontsize{7.3pt}{8.2pt}\selectfont}

\endgroup

\end{document}